\documentclass[letterpaper]{article}
\usepackage{proceed2e}
\usepackage[margin=1in]{geometry}
\usepackage{times,graphicx,subfigure,natbib,algorithm,algorithmic,hyperref}
\usepackage{amsmath,amsthm,amssymb,enumitem,verbatim}

\newcommand{\mE}{\mathbb{E}}
\DeclareMathOperator*{\argmin}{arg\,min}
\newtheorem{lemma}{Lemma}

\newtheorem{corollary}{Corollary}
\graphicspath{{figures/}{../figures/}}

\newcommand\T{\rule{0pt}{2.6ex}}       
\newcommand\B{\rule[-1.2ex]{0pt}{0pt}} 

\title{An Efficient Minibatch Acceptance Test for Metropolis-Hastings}

\author{
Daniel Seita$^1$, 
Xinlei Pan$^1$,
Haoyu Chen$^1$,
John Canny$^{1,2}$ \\
$^1$ University of California, Berkeley, CA \\
$^2$ Google Research, Mountain View, CA\\ 
\texttt{\{seita,xinleipan,haoyuchen,canny\}@berkeley.edu}
}

\begin{document}

\maketitle

\begin{abstract}
  We present a novel Metropolis-Hastings method for large datasets that
  uses small expected-size minibatches of data. Previous work on
  reducing the cost of Metropolis-Hastings tests yield variable data
  consumed per sample, with only constant factor reductions versus
  using the full dataset for each sample.  Here we present a method
  that can be tuned to provide arbitrarily small batch sizes, by
  adjusting either proposal step size or temperature. Our test uses
  the noise-tolerant Barker acceptance test with a novel additive
  correction variable.  The resulting test has similar cost to a normal
  SGD update. Our experiments demonstrate several order-of-magnitude
  speedups over previous work.
\end{abstract}

\section{INTRODUCTION}\label{sec:introduction}

Markov chain Monte Carlo (MCMC) sampling is a powerful method for computation on
intractable distributions. We are interested in large dataset applications,
where the goal is to sample a posterior distribution $p(\theta | x_1, \ldots,
x_N)$ of parameter $\theta$ for large $N$.  The Metropolis-Hastings method (M-H)
generates sample candidates from a proposal distribution $q$ which is in general
different from the target distribution $p$, and decides whether to accept or
reject based on an acceptance test. The acceptance test is usually a Metropolis
test~\citep{Metropolis1953, hastings70}.

Many state-of-the-art machine learning methods, and deep learning in particular,
are based on minibatch updates (such as SGD) to a model.  Minibatch updates
produce many improvements to the model for each pass over the dataset, and have
high sample efficiency.  In contrast, conventional M-H requires calculations
over the full dataset to produce a new sample.  Recent results
from~\citep{cutting_mh_2014} and~\citep{icml2014c1_bardenet14} perform
approximate (bounded error) acceptance tests using subsets of the full dataset.
The amount of data consumed for each test varies significantly from one
minibatch to the next. By contrast,~\citep{conf/uai/MaclaurinA14,TallData16}
perform exact tests but require a lower bound on the parameter distribution across
its domain.  The amount of data reduction depends on the accuracy of this bound,
and such bounds are only available for relatively simple distributions.

Here we derive a new test which incorporates the variability in minibatch
statistics as {\em a natural part of the test} and requires less data per
iteration than prior work. We use a Barker test function~\citep{Barker65}, which
makes our test naturally error tolerant. The idea of using a noise-tolerant
Barker's test function was suggested but not explored empirically
in~\citep{TallData16} section 6.3. But the asymptotic test statistic CDF and the
Barker function are different, which leads to fixed errors for the approach
in~\citep{TallData16}. Here, we show that the difference between the
distributions can be corrected with an additive random variable. This leads to a
test which is fast, and whose error can be made arbitrarily small.

We note that this approach is fundamentally different from prior
work. It makes no assumptions about the form of, and requires no global bounds on the
posterior parameter distribution. It is exact in the
limit as batch size increases by the Central Limit Theorem. This is
not true of~\citep{cutting_mh_2014} and~\citep{icml2014c1_bardenet14}
which use tail bounds and provide only approximate tests even with
arbitrarily large batches of data. Our test is also exact under the
assumptions of~\citet{cutting_mh_2014} that the log probability ratios
of batches are normally distributed about their mean.  Rather
than tail bounds, our approach uses moment estimates from the data to
determine how far the minibatch posteriors deviate from a normal
distribution.  These bounds carry through to the overall accuracy of
the test.

Our test is applicable when the variance (over data samples) of the
log probability ratio between the proposal and the current state is small
enough (less than 1). It's not clear at first why this quantity should
be bounded, but it is natural for well-specified models running
Metropolis-Hastings sampling with optimal
proposals~\citep{OptimalScaling01} on a full dataset. If the posterior
parameter distribution is a unit-variance normal distribution, then
the posterior for $N$ samples will have variance $1/N$. There is
simply not enough information in $M \ll N$ samples to locate and
efficiently sample from this posterior. This is not a property of any
particular proposal or test, but of the information carried by the
data. The variance condition succinctly captures the condition that
the minibatch carries enough information to generate a sample.  While
we cannot expect to generate independent samples from the posterior
using only a small subset of the data, there are three situations
where we can exploit small minibatches:

\begin{enumerate}[noitemsep]
    \item Increase the temperature $K$ of the target distribution. Log
    likelihoods scale as $1/K$, and so the variance of the likelihood ratio will
    vary as $1/K^2$. As we demonstrate in Section~\ref{ssec:logistic}, higher
    temperature can be advantageous for parameter exploration.

    \item For continuous distributions, reduce the proposal step size
      (i.e. generate correlated samples). The variance of the log
      acceptance probability scales as the square of proposal step
      size.

  \item Utilize Hamiltonian Dynamics for proposals and tests. Here the
    dynamics itself provide shaping to the posterior distribution, and
    the M-H test is only needed to correct quantization error. In terms
    of the information carried by the samples, this approach is not
    limited by the data in a particular minibatch since momentum is
    carried over time and ``remembered'' across multiple
    minibatches.
\end{enumerate}

We note that case two above is characteristic of Gibbs samplers applied to
large datasets~\citep{dupuy2016}. Such samplers represent a
model posterior via counts over an entire dataset of $N$ samples. When
a minibatch of $M$ samples is used to update the model, the counts for
these samples only are updated. This creates ``steps'' of $O(M/N)$
in the model parameters, and correlated samples from the model
posterior. Correlated samples are still very useful in high-dimensional
ML problems with multi-modal posteriors since they correspond to a finer-scale
random walk through the posterior landscape.  The contributions of this paper
are as follows:

\begin{itemize}[noitemsep]
    \item We develop a new, more efficient (in samples per test) minibatch
    acceptance test with quantifiable error bounds. The test uses a novel
    additive correction variable to implement a Barker test based on minibatch
    mean and variance. 

    \item We compare our new test and prior approaches on several
      datasets. We demonstrate several order-of-magnitude improvements in sample efficiency,
      and that the batch size distribution is short-tailed. 
\end{itemize}

\section{PRELIMINARIES}\label{sec:related_work}

In the Metropolis-Hastings method~\citep{gilks1996markov,brooks2011handbook}, a
difficult-to-compute probability distribution $p(\theta)$ is sampled using a
Markov chain $\theta_1,\ldots,\theta_T$. The sample $\theta_{t+1}$ at time $t+1$
is generated using a candidate $\theta'$ from a (simpler) proposal distribution
$q(\theta'|\theta_t)$, filtered by an acceptance test. The acceptance test
is usually a Metropolis test. The Metropolis test has acceptance probability:
\begin{equation}\label{eq:traditional}
    \alpha(\theta_t,\theta') = \frac{p(\theta')q(\theta_t | \theta')}{p(\theta_t)q(\theta' | \theta_t)} \wedge 1
\end{equation}
where $a \wedge b$ denotes $\min(a,b)$.  With probability
$\alpha(\theta_t,\theta')$, we accept $\theta'$ and set $\theta_{t+1} =
\theta'$, otherwise set $\theta_{t+1}=\theta_t$.  The test is often implemented
with an auxiliary random variable $u \sim \mathcal{U}(0,1)$ with a comparison
$u<\alpha(\theta_t,\theta')$; here, $\mathcal{U}(a,b)$ denotes the uniform
distribution on the interval $[a,b]$.  For simplicity, we drop the subscript $t$
for the current sample $\theta_t$ and denote it as $\theta$. 

The acceptance test guarantees detailed balance, which means
$p(\theta)p(\theta'|\theta) = p(\theta')p(\theta|\theta')$, where
$p(\theta'|\theta)$ is the probability of a transition from state $\theta$ to
$\theta'$. Here, $p(\theta'|\theta)=q(\theta'|\theta)\alpha(\theta,\theta')$.
This condition, together with ergodicity, guarantees that the Markov chain has a
unique stationary distribution $\pi(\theta) = p(\theta)$. For Bayesian
inference, the target distribution is $p(\theta | x_1, \ldots, x_N)$. The
acceptance probability is now:
\begin{equation}\label{eq:acceptance_probability}
    \alpha(\theta,\theta') = 
    \frac{p_0(\theta')\prod_{i=1}^N p(x_i | \theta')q(\theta |
    \theta')}{p_0(\theta)\prod_{i=1}^N p(x_i | \theta)q(\theta' | \theta)}
    \wedge 1
\end{equation}
where $p_0(\theta)$ is the prior. Computing samples this way requires all $N$
data points, but this is very expensive for large datasets.

To address this
challenge,~\citep{cutting_mh_2014,icml2014c1_bardenet14} perform
approximate Metropolis-Hasting tests using sequential hypothesis
testing. At each iteration, a subset of data is sampled and used to
test whether to accept $\theta'$ using an approximation to
$\alpha(\theta,\theta')$. If the approximate test does not yield a
decision, the minibatch size is increased and the test repeated. This
process continues until a decision. These methods either invoke the
asymptotic CLT and assume that finite  batch errors are normally
distributed~\citep{cutting_mh_2014} or use a concentration
bound~\citep{icml2014c1_bardenet14}. We refer to these algorithms,
respectively, as {\sc AustereMH} and {\sc MHSubLhd}. While both show
useful reductions in the number of samples required, they suffer from
two drawbacks: (i) They are approximate, and always yield a decision
with a finite error, (ii) They both require exact, dataset-wide bounds
that depend on $\theta$ (see Section~\ref{sec:analysis}).\footnote{We obtained
the authors code for both and found that they scanned the entire dataset at each
step to obtain these estimates.} We discuss a
worst-case scenario in Section~\ref{ssec:gaussian_example}.

\subsection{NOTATION}\label{ssec:notation}

Following~\citep{icml2014c1_bardenet14}, we write the test
$u<\alpha(\theta,\theta')$ equivalently as $\Lambda(\theta,\theta') >
\psi(u,\theta,\theta')$, where\footnote{Our definitions differ from those
in~\citep{icml2014c1_bardenet14} by a factor of $N$ to simplify our analysis
later.}
\begin{equation}\label{eq:lambda}
\begin{split}
\Lambda(\theta,\theta') = \sum_{i=1}^N \log\frac{p(x_i|\theta')}{p(x_i|\theta)}, \\
\psi(u,\theta,\theta') = \log\left(u\frac{q(\theta'|\theta)p_0(\theta)}{q(\theta|\theta')p_0(\theta')}\right).
\end{split}
\end{equation}
To simplify notation, we assume that temperature $K=1$ (saving $T$ to indicate
the number of samples to draw). Temperature appears as an exponential on each
likelihood, $p(x_i|\theta)^{1/K}$, so the effect would be to act as a $1/K$
factor on $\Lambda(\theta,\theta')$.

To reduce computational effort, an unbiased estimate of $\Lambda(\theta,\theta')$
based on a minibatch $\{x_1^*,\ldots,x_b^*\}$ can be used:
\begin{equation}
\Lambda^*(\theta,\theta') = \frac{N}{b}\sum_{i=1}^b 
\log \frac{p(x_i^*|\theta')}{p(x_i^*|\theta)}.
\end{equation}
Finally, it will be convenient for our analysis to define
$\Lambda_i(\theta,\theta') = N\log(\frac{p(x_i|\theta')}{p(x_i|\theta)})$.
Thus, $\Lambda(\theta,\theta')$ is the mean of $\Lambda_i(\theta,\theta')$ over
the entire dataset, and $\Lambda^*(\theta,\theta')$ is the mean of
the $\Lambda_i(\theta,\theta')$ in its minibatch. 

Since minibatches contains randomly selected samples, the values $\Lambda_i$ are
i.i.d. random variables.\footnote{The analysis assumes sampling with replacement
although implementations on typical large datasets will approximate this by
sampling without replacement.} By the Central Limit Theorem, we expect
$\Lambda^*(\theta,\theta')$ to be approximately Gaussian. The acceptance test
then becomes a statistical test of the hypothesis that
$\Lambda(\theta,\theta')>\psi(u,\theta,\theta')$ by establishing that
$\Lambda^*(\theta,\theta')$ is substantially larger than
$\psi(u,\theta,\theta')$.

\subsection{A WORST-CASE GAUSSIAN EXAMPLE}\label{ssec:gaussian_example}

Let $x_1,\ldots,x_N$ be i.i.d. $\mathcal{N}(\theta,1)$ with known variance
$\sigma^2=1$ and (unknown) mean $\theta=0.5$. We use a uniform prior on
$\theta$. The log likelihood ratio is
\begin{equation}\label{eq:lemma_ll_ratio}
    \Lambda^*(\theta,\theta') = N(\theta'-\theta)\left(\frac{1}{b}\sum_{i=1}^b
    x_i^*-\theta-\frac{\theta'-\theta}{2}\right)
\end{equation}
which is normally distributed over selection of the Normal samples $x_i^*$.
Since the $x_i^*$ have unit variance, their mean has variance $1/b$, and the
variance of $\Lambda^*(\theta,\theta')$ is $\sigma^2(\Lambda^*) =
(\theta'-\theta)^2N^2/b$.  In order to pass a hypothesis test that $\Lambda >
\psi$, there needs to be a large enough gap (several $\sigma(\Lambda^*)$)
between $\Lambda^*(\theta,\theta')$ and $\psi(u,\theta,\theta')$. 

The posterior is a Gaussian centered on the sample mean $\mu$, and with variance
$1/N$ (i.e., $\mathcal{N}(\mu, 1/N)$). In one dimension, an efficient proposal
distribution has the same variance as the target
distribution~\citep{OptimalScaling01}, so we use a proposal based on
$\mathcal{N}(\theta,1/N)$. It is symmetric
$q(\theta'|\theta)=q(\theta|\theta')$, and since we assumed a uniform prior,
$\psi(u,\theta,\theta')=\log u$. Our worst-case scenario is specified in
Lemma~\ref{lem:worst_case}.

\begin{lemma}\label{lem:worst_case}
    For the model in Section~\ref{ssec:gaussian_example}, there exists a fixed
    (independent of $N$) constant $c$ such that with probability $\geq c$ over
    the joint distribution of $(\theta, \theta', u)$, {\sc AustereMH} and {\sc
    MHSubLhd} consume all $N$ samples. 
\end{lemma}
\vspace{-1em}
\begin{proof}
See Appendix, Section~\ref{app:worst_case_proof}.
\end{proof}
Similar results can be shown for other distributions and proposals by
identifying regions in product space $(\theta,\theta',u)$ such that the
hypothesis test needs to separate nearly-equal values.  It follows that the
accelerated tests from prior work require at least a constant fraction $\geq c$
in the amount of data consumed per test compared to full-data tests, so their
speed-up is $\le 1/c$. The issue is the use of tail bounds to separate $\Lambda-\psi$
from zero; for certain input/random $u$ combinations, this difference can be
arbitrarily close to zero. We avoid this by using the {\em approximately normal}
variation in $\Lambda^*$ to {\em replace} the variation due to $u$. 

\subsection{MCMC POSTERIOR INFERENCE}
There is a separate line of MCMC work drawing principles from statistical
physics. One can apply Hamiltonian Monte Carlo
(HMC)~\citep{mcmc_hamiltonian_2010} methods which generate high acceptance
\emph{and} distant proposals when run on full batches of data. Recently Langevin
Dynamics~\citep{langevin_2011,conf/icml/AhnBW12} has been applied to Bayesian
estimation on minibatches of data. This simplified dynamics uses local proposals
and avoids M-H tests by using small proposal steps whose acceptance approaches 1
in the limit. However, the constraint on proposal step size is severe, and the
state space exploration reduces to a random walk.  Full minibatch HMC for
minibatches was described in~\citep{sghmc_2014} which allows momentum-augmented
proposals with larger step sizes. However, step sizes are still limited by the
need to run accurately without M-H tests.  By providing an M-H test with similar
cost to standard gradient steps, our work opens the door to applying those
methods with much more aggressive step sizes without loss of accuracy.

\section{A NEW MH ACCEPTANCE TEST}\label{sec:our_algorithm}

\subsection{LOG-LIKELIHOOD RATIOS}\label{ssec:log_likelihood_ratios}

For our new M-H test, we denote the exact and approximate log likelihood ratios
as $\Delta$ and $\Delta^*$, respectively. First, $\Delta$ is defined as
\begin{equation}\label{eq:delta1}
    \Delta(\theta,\theta')  =
    \log \frac{p_0(\theta')\prod_{i=1}^N p(x_i | \theta')q(\theta |
    \theta')}{p_0(\theta)\prod_{i=1}^N p(x_i | \theta)q(\theta' | \theta)},
\end{equation}
where $p_0, p$, and $q$ match the corresponding functions within
Equation~(\ref{eq:acceptance_probability}). We separate out terms dependent and
independent of the data as:
\begin{equation}\label{eq:delta2}
\Delta(\theta,\theta') =
\underbrace{\sum_{i=1}^N\log\frac{p(x_i | \theta')}{p(x_i | \theta)}}_{\Lambda(\theta,\theta')}
- \psi(1,\theta,\theta').
\end{equation}
A minibatch estimator of $\Delta$, denoted as $\Delta^*$, is
\begin{equation}\label{eq:delta3}
\Delta^*(\theta,\theta') =
\underbrace{\frac{N}{b}\sum_{i=1}^b\log\frac{p(x_i^* | \theta')}{p(x_i^* | \theta)}}_{\Lambda^*(\theta,\theta')}
- \psi(1,\theta,\theta').
\end{equation}
Note that $\Delta$ and $\Delta^*$ are evaluated on the full dataset and a
minibatch of size $b$ respectively. The term $N/b$ means
$\Delta^*(\theta,\theta')$ is an unbiased estimator of $\Delta(\theta,\theta')$.

The key to our test is a smooth acceptance function.  We consider functions
other than the classical Metropolis test that satisfy the detailed balance
condition needed for accurate posterior estimation. A class of suitable
functions is specified as follows:

\begin{lemma}\label{lem:detailed_balance}
    If $g(s)$ is any function such that $g(s) = \exp(s) g(-s)$, then the
    acceptance function $\alpha(\theta,\theta') \triangleq
    g(\Delta(\theta,\theta'))$ satisfies detailed balance.
\end{lemma}

This result is used in~\citep{Barker65} to define the Barker acceptance test.

\subsection{BARKER (LOGISTIC) ACCEPTANCE FUNCTION}\label{ssec:barker_function}
For our new MH test we use the Barker logistic~\citep{Barker65}
function: $g(s)=(1+\exp(-s))^{-1}$. Straightforward arithmetic shows
that it satisfies the condition in Lemma~\ref{lem:detailed_balance}.
It is slightly less efficient than the Metropolis test, since its
acceptance rate for vanishing likelihood difference is 0.5. However we
will see that its overall sample efficiency is much higher than the
earlier methods. See Appendix~\ref{app:why_barker} for additional discussion.

Assume we begin with the current sample $\theta$ and a candidate sample
$\theta'$, and that $V \sim \mathcal{U}(0,1)$ is a uniform random variable. We
accept $\theta'$ if $g(\Delta(\theta,\theta')) > V$, and reject otherwise.
Since $g(s)$ is monotonically increasing, its inverse $g^{-1}(s)$ is
well-defined and unique. So an equivalent test is to accept $\theta'$ iff
\begin{equation}\label{eq:equivalent_test}
    \Delta(\theta,\theta') > X = g^{-1}(V)
\end{equation}
where $X$ is a random variable with the logistic distribution (its CDF is the
logistic function). To see this notice that $\frac{dV}{dX} = g'$, that $g'$ is
the density corresponding to a logistic CDF, and finally that $\frac{dV}{dX}$ is
the density of $X$. The density of $X$ is symmetric, so we can equivalently test
whether
\begin{equation}\label{eq:the_exact_test}
    \Delta(\theta,\theta') + X > 0
\end{equation}
for a logistic random variable $X$.

\subsection{A MINIBATCH ACCEPTANCE TEST}\label{ssec:deltas_minibatch}

We now describe acceptance testing using the minibatch estimator
$\Delta^*(\theta,\theta')$. From Equation~(\ref{eq:delta3}),
$\Delta^*(\theta,\theta')$ can be represented as a constant term plus the mean
of $b$ IID terms $\Lambda_i(\theta,\theta')$ of the form
$N\log\frac{p(x_i^*|\theta')}{p(x_i^*|\theta)}$. As $b$ increases,
$\Delta^*(\theta,\theta')$ therefore has a distribution which approaches a
normal distribution by the Central Limit Theorem. We now describe this using an
asymptotic argument and defer specific bounds between the CDFs of
$\Delta^*(\theta,\theta')$ and a Gaussian to Section~\ref{sec:analysis}.

In the limit, since $\Delta^*$ is normally distributed about its mean $\Delta$,
we can write
\begin{equation}\label{eq:relationship}
    \Delta^* = \Delta + X_{\rm norm}, \quad X_{\rm norm} \sim \mathcal{\bar{N}}(0, \sigma^2(\Delta^*)),
\end{equation}
where $\mathcal{\bar{N}}(0, \sigma^2(\Delta^*))$ denotes a distribution which is
approximately normal with variance $\sigma^2(\Delta^*)$.  But to perform the
test in Equation~(\ref{eq:the_exact_test}) we want $\Delta + X$ for a logistic
random variable $X$ (call it $X_{\rm log}$ from now on). In~\citep{TallData16} it
was proposed to use $\Delta^*$ in a Barker test, and tolerate the fixed
error between the logistic and normal distributions. 

Our approach is to instead decompose $X_{\rm log}$ as
\begin{equation}\label{eq:deconvolution}
    X_{\rm log} = X_{\rm norm}+X_{\rm corr},
\end{equation}
where we assume $X_{\rm norm} \sim \mathcal{N}(0, \sigma^2)$ and that $X_{\rm
corr}$ is a zero-mean ``correction'' variable with density $C_{\sigma}(X)$.  The
two variables are added (i.e., their distributions convolve) to form $X_{\rm
log}$.  This decomposition requires an appropriate $C_\sigma$, which we derive
in Section~\ref{sec:correction}. Using $X_{\rm corr}$ samples from
$C_{\sigma}(X)$, the acceptance test is now
\begin{equation}\label{eq:criteria}
    \Delta + X_{\rm log} = (\Delta + X_{\rm norm}) + X_{\rm corr} = \Delta^* + X_{\rm corr} >0.
\end{equation}
Therefore, assuming the variance of $\Delta^*$ is small enough, if we have an
estimate of $\Delta^*$ from the current data minibatch, we test acceptance by
adding a random variable $X_{\rm corr}$ and then accept $\theta'$ if the result
is positive (and reject otherwise).

If $\mathcal{\bar{N}}(0, \sigma^2(\Delta^*))$ is exactly $\mathcal{N}(0,
\sigma^2(\Delta^*))$, the above test is exact, and as we show in
Section~\ref{sec:analysis}, if there is a maximum error $\epsilon$ between the
CDF of $\mathcal{\bar{N}}(0, \sigma^2(\Delta^*))$ and the CDF of $\mathcal{N}(0,
\sigma^2(\Delta^*))$, then our test has an error of at most $\epsilon$ relative
to the full batch version.

\section{THE CORRECTION DISTRIBUTION}\label{sec:correction}

Our test in Equation~(\ref{eq:criteria}) requires knowing the distribution of
$X_{\rm corr}$. In Section~\ref{sec:analysis}, we show that the test accuracy
depends on the absolute error between the CDFs of $X_{\rm norm} + X_{\rm corr}$
and $X_{\rm log}$. Consequently, we need to minimize this in our construction of
$X_{\rm corr}$. More formally, let $\Phi_{s_X} = \Phi(X/s_X)$ where $\Phi$ is
the standard normal CDF\footnote{Hence, $\Phi_{s_X}$ is the CDF of a zero-mean
Gaussian with standard deviation $s_X$.}, $S(X)$ be the logistic function, and
$C_{\sigma}(X)$ be the \emph{density} of the correction $X_{\rm corr}$
distribution. Our goal is to solve:
\begin{equation}\label{eq:overall_corr_problem}
    C_\sigma^* = \argmin_{C_\sigma} |\Phi_{\sigma} * C_{\sigma} - S|
\end{equation}
where $*$ denotes convolution. To compute $C_\sigma$, we assume the input $Y$
and another variable $X$ lie in the intervals $[-V,V]$ and $[-2V,2V]$,
respectively.  We discretize the convolution by discretizing $X$ and $Y$ into
$4N+1$ and $2N+1$ values respectively. If $i \in \{-2N, \ldots,
2N\}=\mathcal{I}$ and $j \in \{-N, \ldots, N\}=\mathcal{J}$, then we can write
$X_i = i(V/N)$ and $Y_j = j(V/N)$, and the objective can be written as:
\[
C_\sigma^* = \argmin_{C_\sigma} \max_{i \in \mathcal{I}}\left|\sum_{j\in\mathcal{J}} \Phi_{\sigma}(X_i-Y_j) C_{\sigma}(Y_j) - S(X_i)\right|.
\]
Now define matrix $M$ and vectors $u$ and $v$ such that $M_{ij} =
\Phi_{\sigma}(X_i-Y_j)$, $u_j = C_{\sigma}(Y_j)$, and $v_i = S(X_i)$, where the
indices $i$ and $j$ are appropriately translated to be non-negative for $M, u,$
and $v$. The problem is now to minimize $\|Mu-v\|_{\infty}$ with the density
non-negative constraint $u > 0$. We approximate this with least squares:
\begin{equation}\label{eq:optimization_l2}
    u^* = \argmin_u\; \|Mu-v\|_2^2 + \lambda \|u\|_2^2,
\end{equation}
with regularization $\lambda$. The solution is well-known from the normal
equations ($u^* = (M^TM + \lambda I)^{-1}M^Tv$) and in practice yields an
acceptable $L_{\infty}$ norm.

\begin{algorithm}[tb]
\caption{{\sc MHminibatch} acceptance test.}
\label{alg:our_algorithm}
\begin{algorithmic}
    \STATE {\bfseries Input:} number of samples $T$, minibatch size $m$, error
        bound $\delta$, pre-computed correction $C_1(X)$ distribution, initial
        sample $\theta_1$.
    \STATE {\bfseries Output:} a chain of $T$ samples
        $\{\theta_1,\ldots,\theta_T\}$.
    \FOR{$t=1$ {\bfseries to} $T$}
        \STATE -Propose a candidate $\theta'$ from proposal
            $q(\theta' | \theta_t)$.
        \STATE -Draw a minibatch of $m$ points $\{x_1^*,\ldots,x_m^*\}$.
        \STATE -Compute $\Delta^*(\theta_t,\theta')$ and sample variance
            $s^2_{\Delta^*}$.
        \STATE -Estimate moments $\mE[|\Lambda_i-\Lambda|]$ and
            $\mE[|\Lambda_i-\Lambda|^3]$ from the sample, and error $\epsilon$
            from Corollary~\ref{cor:our_bound_delta_prime}.
        \WHILE{$s^2_{\Delta^*} \ge 1$ {\bf or} $\epsilon > \delta$}
            \STATE -Draw $m$ more samples to augment the minibatch, update
                $\Delta^*$, $s^2_{\Delta^*}$ and $\epsilon$ estimates.
        \ENDWHILE
        \STATE -Draw $X_{\rm nc} \sim \mathcal{N}(0,1-s^2_{\Delta^*})$ and
            $X_{\rm corr} \sim C_1(X)$.
    \IF{$\Delta^* + X_{\rm nc} + X_{\rm corr} > 0$}
        \STATE -Accept the candidate, $\theta_{t+1} = \theta'$.
    \ELSE
        \STATE -Reject and re-use the old sample, $\theta_{t+1} = \theta_t$.
    \ENDIF
    \ENDFOR
\end{algorithmic}
\end{algorithm}

With this approach, there is no guarantee that $u^* \geq 0$. However, we have
some flexibility in the choice of $\sigma$ in
Equation~(\ref{eq:overall_corr_problem}).  As we decrease the variance of
$X_{\rm norm}$, the variance of $X_{\rm corr}$ grows by the same amount and is
in fact the result of convolution with a Gaussian whose variance is the
difference.  Thus as $\sigma$ decreases, $C_\sigma(X)$ grows and approaches the
derivative of a logistic function at $\sigma = 0$. It retains some weak negative
values for $\sigma > 0$ but removal of those leads to small error. We use
$N=4000$ and $\lambda=10$ for our experiments, which empirically provided
excellent performance. See Table~\ref{tab:xcorr} in
Appendix~\ref{app:correction} for detailed $L_{\rm \infty}$ errors for different
settings. Algorithm~\ref{alg:our_algorithm} describes our procedure, {\sc
MHminibatch}. A few points:

\begin{itemize}[noitemsep]
    \item It uses an adaptive step size so as to use the smallest possible
    average minibatch size. Unlike previous work, the size distribution is
    short-tailed.

    \item An additional normal variable $X_{\rm nc}$ is added to $\Delta^*$ to
    produce a variable with unit variance. This is not mathematically necessary,
    but allows us to use a single correction distribution $C_1$ with $\sigma=1$
    for $X_{\rm corr}$, saving on memory footprint.

    \item The sample variance of $\Delta^*$ is denoted as $s^2_{\Delta^*}$ and
    is proportional to $\|\theta'-\theta\|_2^2$.
    
\end{itemize}

\section{ANALYSIS}\label{sec:analysis}

We now derive error bounds for our M-H test and the target distribution it
generates. 
In Section~\ref{ssec:delta_star_distribution}, we present bounds on the absolute
and relative error (in terms of the CDFs) of the distribution of $\Delta^*$
versus a Gaussian. We then show in Section~\ref{ssec:preserve_bounds} that these
bounds are preserved after the addition of other random variables (e.g., $X_{\rm
nc}$ and $X_{\rm corr}$). It then follows that the acceptance test has the same
error bound.

\subsection{BOUNDING THE ERROR OF $\Delta^*$ FROM A GAUSSIAN}\label{ssec:delta_star_distribution}

We use the following quantitative central-limit result:
  
\begin{lemma}\label{lem:quant_clt}
Let $X_1,\ldots,X_n$ be a set of zero-mean, independent, identically-distributed
random variables with sample mean $\bar{X}$ and sample variance $s^2_X$
where:
\begin{equation}
    \bar{X} = \frac{1}{n}\sum_{i=1}^nX_i, \quad s_X = \frac{1}{n}\left(\sum_{i=1}^n(X_i-\bar{X})^2\right)^{\frac{1}{2}}.
\end{equation}
Then the t-statistic $t=\bar{X}/s_X$ has a distribution which is approximately normal, with
error bounded by:
\begin{equation}\label{eq:clt-bounds}
    \sup_x|{\rm Pr}(t<x) - \Phi(x)| \leq \frac{6.4\mE[|X|^3]+2\mE[|X|]}{\sqrt{n}}.
\end{equation}
\end{lemma}

\begin{proof}
See Appendix, Section~\ref{app:quant_clt}.
\end{proof}

Lemma~\ref{lem:quant_clt} demonstrates that if we know $\mE[|X|]$ and
$\mE[|X|^3]$, we can bound the error of the normal approximation, which decays
as $O(n^{-\frac{1}{2}})$. Making the change of variables $y = x s_X$,
Equation~(\ref{eq:clt-bounds}) becomes
\begin{equation}\label{eq:clt-bounds2}
   \sup_y\left|{\rm Pr}(\bar{X}< y) - \Phi\left(\frac{y}{s_X}\right)\right| \leq
   \frac{6.4\mE[|X|^3]+2\mE[|X|]}{\sqrt{n}}
\end{equation}
showing that the distribution of $\bar{X}$ approaches the normal distribution
$\mathcal{N}(0,s_X)$ whose standard deviation is $s_X$, as measured
from the sample.

To apply this to our test, let $X_i = \Lambda_i(\theta,\theta') -
\Lambda(\theta,\theta')$, so that the $X_i$ are zero-mean, i.i.d. variables. If
instead of all $n$ samples, we only extract a subset of $b$ samples
corresponding to our minibatch, we can connect $\bar{X}$ with our $\Delta^*$
term: $\bar{X} = \Delta^*(\theta,\theta') - \Delta(\theta,\theta')$, so that
$s_X = s_{\Delta^*}$. We can now substitute into Equation~(\ref{eq:clt-bounds2})
and displace by the mean, giving:

\begin{corollary}\label{cor:our_bound_delta_prime}
\small
\begin{equation}
\sup_y \left|{\rm Pr}(\Delta^*< y) -\Phi\left(\frac{y-\Delta}{s_{\Delta^*}}\right)\right| \leq \frac{6.4\mE[|X|^3]+2\mE[|X|]}{\sqrt{b}}
\end{equation}
\normalsize
\end{corollary} 
where the upper bound can be expressed as $\epsilon(\theta,\theta',b)$.
Corollary~\ref{cor:our_bound_delta_prime} shows that the distribution of
$\Delta^*$ approximates a Normal distribution with mean $\Delta$ and variance
$s^2_{\Delta^*}$. Furthermore, it bounds the error with \emph{estimable
quantities}: both $\mE[|X|]$ and $\mE[|X|^3]$ can be estimated as means of
$|\Lambda_i - \Lambda|$ and $|\Lambda_i - \Lambda|^3$, respectively, on each
minibatch. We expect this will often be accurate enough on minibatches with
hundreds of points, but otherwise bootstrap CIs can be computed.

\subsection{ADDING RANDOM VARIABLES}\label{ssec:preserve_bounds}

We next relate the CDFs of distributions and show that bounds are preserved
after adding random variables.

\begin{lemma}\label{lem:cdf_bounds}
Let $P(x)$ and $Q(x)$ be two CDFs satisfying
$\sup_x|P(x)-Q(x)|\leq \epsilon$ with $x$ in some real range. Let $R(y)$ be the
{\em density} of another random variable $y$. Let $P'$ be the convolution $P*R$
and $Q'$ be the convolution $Q*R$. Then $P'(z)$ (resp. $Q'(z)$) is the CDF of
sum $z=x+y$ of independent random variables $x$ with CDF $P(x)$ (resp. $Q(x)$)
and y with density $R(y)$.  Then
\begin{equation}
    \sup_x|P'(x)-Q'(x)|\leq \epsilon.
\end{equation}
\end{lemma}
\begin{proof}
See Appendix, Section~\ref{app:proof_cdf_bounds}.
\end{proof}

From Lemma~\ref{lem:cdf_bounds}, we have the following Corollary:

\begin{corollary}\label{cor:bounds_preserved}
If $\sup_y|{\rm Pr}(\Delta^* < y) - \Phi(\frac{y-\Delta}{s_{\Delta^*}})|
\leq \epsilon(\theta,\theta',b)$, then
\[
    \sup_y|{\rm Pr}(\Delta^*+X_{\rm nc}+X_{\rm corr} < y) - S(y-\Delta)| \leq \epsilon(\theta,\theta',b)
\]
where $S(x)$ is the standard logistic function, and $X_{\rm nc}$ and $X_{\rm
corr}$ are generated as per Algorithm 1.
\end{corollary}

\begin{proof}
See Appendix, Section~\ref{app:bounds_preserved}.
\end{proof}

Corollary~\ref{cor:bounds_preserved} shows that the bounds from
Section~\ref{ssec:delta_star_distribution} are preserved after adding random
variables, so our test remains accurate.  In fact we can do better ($O(n^{-1})$
instead of $O(n^{-1/2})$) by using a more precise limit distribution under an
additional assumption. We review this in Appendix~\ref{app:better_error_bound}.

\subsection{BOUNDS ON THE STATIONARY DISTRIBUTION}
Bounds on the error of an M-H test imply bounds on the stationary distribution
of the Markov chain under appropriate conditions. Such bounds were derived in
both~\citep{cutting_mh_2014} and~\citep{icml2014c1_bardenet14}. We include the
result from~\citep{cutting_mh_2014} (Theorem 1) here: Let $d_v(P,Q)$ denote the
total variation distance between two distributions $P$ and $Q$.  Let
$\mathcal{T}_0$ denote the transition kernel of the exact Markov chain,
$\mathcal{S}_0$ denote the exact posterior distribution, and
$\mathcal{S}_{\epsilon}$ denote the stationary distribution of the approximate
transition kernel. 
\begin{lemma}
If $\mathcal{T}_0$ satisfies the contraction condition
$d_v(P\mathcal{T}_0,\mathcal{S}_0) < \eta d_v(P,\mathcal{S}_0)$ for some
constant $\eta\in [0,1)$ and all probability distributions $P$, then
\begin{equation}
      d_v(S_0, S_{\epsilon}) \leq\frac{\epsilon}{1-\eta}
\end{equation}
where $\epsilon$ is the bound on the error in the acceptance test. 
\end{lemma}

\section{EXPERIMENTS}\label{sec:experiments}
Here we compare with the most similar prior works
~\citep{cutting_mh_2014} and~\cite{icml2014c1_bardenet14}.
In~\citep{cutting_mh_2014}, an asymptotic CLT is used to argue that a
modified standard M-H test can be used on subsets of the data. This
assumes knowledge of dataset-wide mean $\mu_{\rm std}$ each iteration
(it depends on $\theta$). Determining $\mu_{\rm std}$ exactly requires
a scan over the entire dataset, or some model-specific bounds.
\citep{cutting_mh_2014} also propose a conservative variant which
assumes $\mu_{\rm std}=0$ and avoids the scan.  We refer to the
conservative version as {\sc AustereMH(c)} and the non-conservative
variant as {\sc AustereMH(nc)}. We analyze both in this section.

In~\citep{icml2014c1_bardenet14}
concentration bounds are used with a similar modification to the standard M-H
test ({\sc MHSubLhd} method).  For {\sc MHSubLhd}, the required global
bound is denoted $C_{\theta, \theta'}$ which once again depends on
$\theta$ and so must be recomputed at each step, or estimated in a
model-specific way.   We obtained
sample code for both methods from the authors, and found that both
{\sc AustereMN(nc)} and {\sc MHSubLhd} scanned the entire dataset at
each iteration to derive these bounds. We do not include the cost of
doing this in our experiments, since otherwise there would be no
improvement over testing the full dataset. However, it should be kept
in mind that such bounds must be provided to these methods.
Our test by contrast uses a quantitative form of the CLT
which rely on measurable statistics from a \emph{single} minibatch.
It therefore requires no dataset-wide scans, and can be used, e.g. on
streams of data. 

\begin{figure*}[t]
    \centering
    \includegraphics[width=1.0\linewidth]{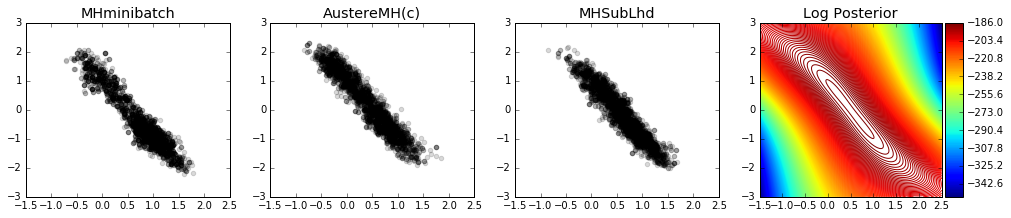}
    \caption{
    The log posterior contours and scatter plots of sampled $\theta$ values
    using different methods. 
    }
    \label{fig:gauss_mix_1}
\end{figure*}
\begin{figure*}[t]
    \centering
    \includegraphics[width=0.9\linewidth]{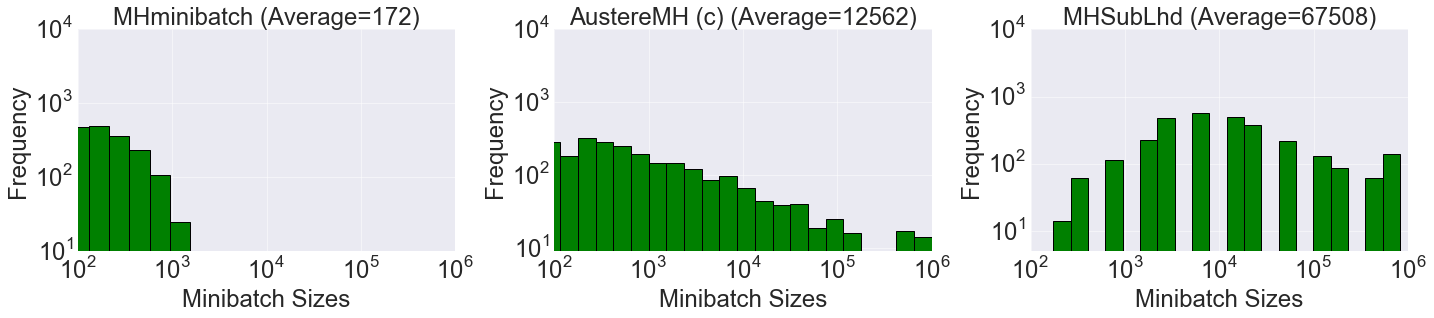}
    \caption{
    Minibatch sizes used in Section~\ref{ssec:gaussians}'s experiment. The axes
    have the same (log-log scale) range.
    }
    \label{fig:gauss_mix_2}
\end{figure*}

In Sections~\ref{ssec:gaussians} and~\ref{ssec:logistic}, we benchmark {\sc
MHminibatch} against {\sc MHSubLhd}, {\sc AustereMH(c)} and {\sc AustereMH(nc)}.
Hyperparameters for the latter were optimized using a grid-search over minibatch
sizes $m$ and per-test thresholds $\epsilon$ described in
Appendix~\ref{sssec:grid_search}. Throughout our descriptions, we refer to a
\emph{trial} as the period when an algorithm collects all its desired samples
$\{\theta_1,\ldots,\theta_T\}$, generally with $T=3000$ or $T=5000$.

\subsection{MIXTURE OF GAUSSIANS}\label{ssec:gaussians}

This model is adapted from~\citep{langevin_2011} by increasing the number of
samples to 1 million.  The parameters are $\theta = \langle \theta_1,\theta_2
\rangle$, and the generation process is
\begin{equation}\label{eq:data_generation}
\begin{split}
    \theta &\sim \mathcal{N}(0, {\rm diag}(\sigma_1^2,\sigma_2^2)) \\
    x_i & \sim 0.5 \cdot \mathcal{N}(\theta_1, \sigma_x^2) + 0.5 \cdot \mathcal{N}(\theta_1+\theta_2, \sigma_x^2).
\end{split}
\end{equation}
We set $\sigma_1^2 = 10, \sigma_2^2 = 1$ and $\sigma_x^2=2$.  We fix $\theta =
\langle 0,1 \rangle$. The original paper sampled 100 data points and estimated
the posterior. We are interested in performance on larger problems and so
sampled 1,000,000 points to form the posterior of
$p(\theta)\prod_{i=1}^{1,000,000}p(x_i | \theta)^{1/K}$ with the same prior from
Equation~(\ref{eq:data_generation}). This produces a much sharper posterior with
two very narrow peaks.  Our goal is to reproduce the original posterior, so we
adjust the temperature to $K=10,000$.  Taking logs, we get the target as shown
in the far right of Figure~\ref{fig:gauss_mix_1}.

We benchmark with {\sc AustereMH(c)} and {\sc MHSubLhd}. We initialized {\sc
MHminibatch} and {\sc MHSubLhd} with $m=50$. For {\sc AustereMH(c)}, we set
the error bound $\epsilon$ to 0.005.  For {\sc MHSubLhd}, we increase
sizes geometrically with $\gamma = 1.5$ and use parameters $p = 2, \delta =
0.01$.  All methods collect 3000 samples using a random walk proposer with
covariance matrix ${\rm diag}(0.15, 0.15)$, which means the M-H test is
responsible for shaping the sample distribution.

Figure~\ref{fig:gauss_mix_1} shows scatter plots of the resulting $\theta$
samples for the three methods, with darker regions indicating a greater density
of points. There are no obvious differences, showing that {\sc MHminibatch}
reaches an acceptable posterior. We further measure the similarity between each
set of samples and the actual posterior. Due to space constraints, results are
in Appendix~\ref{sssec:gauss_metrics}.

Figure~\ref{fig:gauss_mix_2} shows that {\sc MHminibatch} dominates in terms of
speed and efficiency. The histograms of the (final) minibatch sizes used each
iteration show that our method consumes significantly less data; the
distribution is short-tailed and the mean is 172, more than an order of
magnitude better compared to the other two methods (averages are 12562 and
67508). We further ran 10 runs of mixture of Gaussians experiments and report
minibatch sizes in Table~\ref{tab:gaussianres}. Sizes correspond to the running
times of the methods, excluding the likelihood computation of all data points
for {\sc AustereMH(nc)} and {\sc MHSubLhd}, which would drastically increase
    running time.

\begin{table}[t]
\caption{Average minibatch sizes ($\pm$ one standard deviation) on the Gaussian
mixture model. The averages are taken over 10 independent trials (3000 samples
each).}
\label{tab:gaussianres}
\vskip 0.15in
\begin{center}
\begin{tabular}{l l l l}
\textbf{Method} & \textbf{Average of MB Sizes} \\
\hline \\
{\sc MHminibatch}   & $182.3\pm 11.4$ \\
{\sc AustereMH(c)}  & $13540.5\pm 1521.4$ \\
{\sc MHSubLhd}      & $65758.9\pm 3222.6$ \\
\end{tabular}
\end{center}
\vskip -0.1in
\end{table}

\subsection{LOGISTIC REGRESSION}\label{ssec:logistic}

\begin{figure*}[t]
	\centering
	\includegraphics[width=1.0\linewidth]{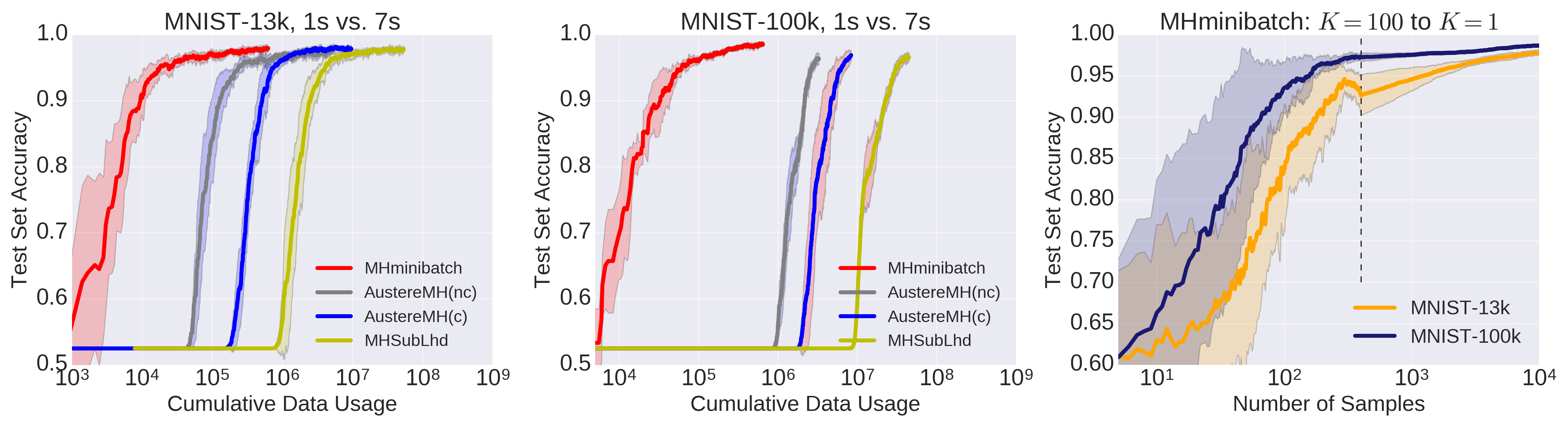}
	\caption{
    Binary classification accuracy of the MCMC methods on the 1s vs 7s logistic
    regression task for MNIST-13k (left plot) and MNIST-100k (middle plot) as a
    function of cumulative data usage.  The right plot reports performance of
    {\sc MHminibatch} on both datasets when the temperature starts at 100 and
    drops to 1 after a ``burn-in'' period of 400 samples (vertical dashed line)
    of $\theta$.  For all three plots, one standard deviation is indicated by
    the shaded error regions.
    }
	\label{fig:logistic_performance}
\end{figure*}

\begin{figure*}[t]
	\centering
    \includegraphics[width=0.9\linewidth]{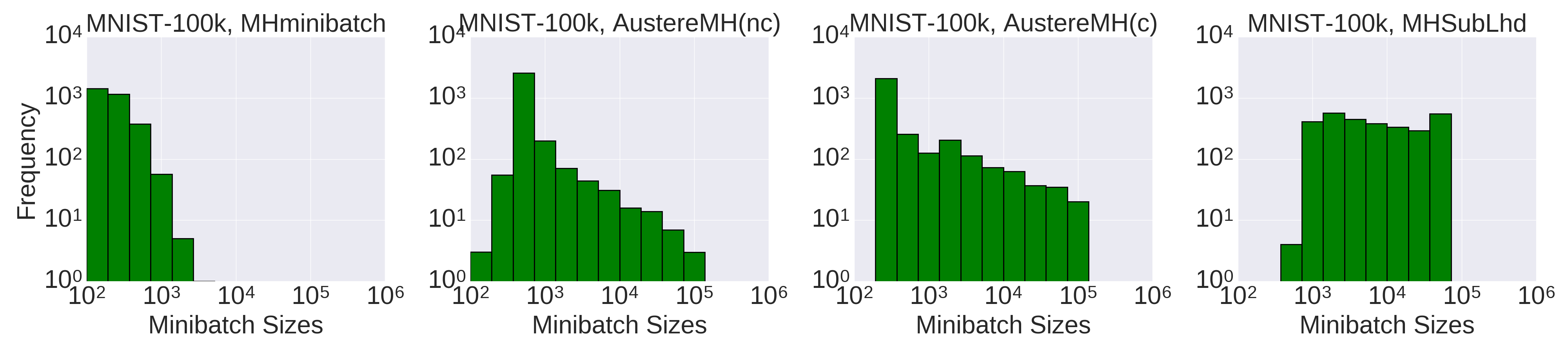}
	\caption{
    Minibatch sizes for a representative trial of logistic regression on
    MNIST-100k (analogous to Figure~\ref{fig:gauss_mix_2}). Both axes are on a
    log scale and have the same ranges across the three histograms.  See
    Section~\ref{ssec:logistic} for details.
    }
	\label{fig:logistic_histograms_mnist8m}
\end{figure*}

We next test logistic regression for the binary classification of 1s versus 7s
on the MNIST~\citep{lecun-mnisthandwrittendigit-2010} dataset and (a subset of)
infinite MNIST~\citep{loosli-canu-bottou-2006}. For the former, extracting all 1s
and 7s resulted in 13,000 training samples, and for the latter, we used 87,000
additional (augmented) 1s and 7s to get 100,000 training samples. Both datasets
use the same test set, with 2,163 samples. Henceforth, we call them MNIST-13k
and MNIST-100k, respectively.

For all methods, we impose a uniform prior on $\theta$ and again use a random
walk proposer, with covariance matrix $0.05I$ for MNIST-13k and $0.01I$ for
MNIST-100k. The default temperature setting is a constant at $K=100$ for
MNIST-13k and MNIST-100k. Performance of all methods implicitly relies on the
step size and temperature. Setting temperature too low or step size too high
will result in slow convergence for all methods. For MNIST-13k, each method
generated 5000 samples for ten independent trials; due to MNIST-100k's higher
computational requirements, the methods generated 3000 samples for five
independent trials.  For additional parameter settings and an investigation on
tuning step sizes, see Appendix~\ref{app:logistic}.

For {\sc MHSubLhd}, we tried to use the provided symbolic bound for
$C_{\theta,\theta'}$ described in~\citep{icml2014c1_bardenet14}, but it was too
high and provided no performance benefit. Instead we use the empirical
$C_{\theta,\theta'}$ from the entire dataset. 

The first two subplots of Figure~\ref{fig:logistic_performance} display the
prediction accuracy on both datasets for all methods as a function of the
cumulative training points processed.\footnote{The curves do not span the same
length over the x-axis since the methods consume different amounts of data.} To
generate the curves, for each of the sampled vectors $\theta_t$,
$t\in\{1,\ldots,T\}$, we use $\theta_t$ as the logistic regression parameter.
The results indicate that our test is more efficient, obtaining convergence more
than an order of magnitude faster than {\sc AustereMH(nc)} and several orders of
magnitude compared to {\sc AustereMH(c)} and {\sc MHSubLhd}.  We also observe
the advantage of having higher temperature from the third plot in
Figure~\ref{fig:logistic_performance}, which plots average performance and one
standard deviation for {\sc MHminibatch} over 10 trials. During the exploration
period, the accuracy rapidly increases, and then after 400 samples, we switch
the temperature to 1, but this requires the step size to decrease, hence the
smaller changes in accuracy.

Figure~\ref{fig:logistic_histograms_mnist8m} shows log-log histograms of
minibatch sizes for the methods on MNIST-100k.
(Figure~\ref{fig:logistic_histograms_mnist} in Appendix~\ref{app:logistic}
contains results for MNIST-13k.) The histograms only represent one
representative trial; Table~\ref{tab:logistic} contains the average of the
average minibatch sizes ($\pm$ one standard deviation) across all trials. {\sc
MHminibatch}, with average minibatch sizes of 125.4 and 216.5 for MNIST-13k and
MNIST-100k, respectively, consumes more than 7x and 4x fewer data points than
the next-best method, {\sc AustereMH(nc)}.  We reiterate, however, that both
{\sc AustereMH(nc)} and {\sc MHSubLhd} require computing $\log p(x_i|\theta)$
and $\log p(x_i|\theta')$ for all $x_i$ each iteration. Our results here do not
count that extra data consumption. Only our method and {\sc AustereMH(c)} rely
solely on the minibatch of data each iteration.

\begin{table}[t]
\caption{Average minibatch sizes ($\pm$ one standard deviation) on logistic
regression on MNIST-13k and MNIST-100k.  The averages are taken over 10
independent trials (5000 samples each) for MNIST-13k and 5 independent trials
(3000 samples each) for MNIST-100k.}
\small 
\label{tab:logistic}
\vskip 0.15in
\begin{center}
\begin{tabular}{l l l l}
\textbf{Method/Data} & \textbf{MNIST-13k} & \textbf{MNIST-100k}  \\
\hline \\
{\sc MHminibatch}   & $125.4\pm 9.2$    & $216.5 \pm 7.9$ \\
{\sc AustereMH(nc)} & $973.8\pm 49.8$   & $1098.3 \pm 44.9$ \\
{\sc AustereMH(c)}  & $1924.3\pm 52.4$  & $2795.6 \pm 364.0$ \\
{\sc MHSubLhd}      & $10783.4\pm 78.9$ & $14977.3 \pm 582.0$ \\
\end{tabular}
\end{center}
\vskip -0.1in
\end{table}

\section{CONCLUSIONS AND DISCUSSIONS}\label{sec:conclusion}

We have derived an M-H test for minibatch MCMC which approximates full data
tests. We present theoretical results and experimentally show the benefits of
our test on Gaussian mixtures and a logistic regression experiment.

A priority is to extend our work to methods such as Hamiltonian Monte Carlo and Langevin Dynamics
which use efficient but asymmetric proposals. While there are various approaches
to symmetrizing these proposals, they have high cost in the context of minibatch
MCMC. Instead we plan to extend our method to log proposal ratios which have
similar structure (whole-dataset mean plus additive noise) to the log probability
ratio. These can be similarly absorbed in the Barker test. 

Other possibilities for future work include integrating our algorithm
with~\citep{cutting_mh_2014} by applying both tests each iteration, utilizing
the variance reduction techniques suggested in~\citep{DBLP:conf/icml/ChenG16},
and providing recipe for how to use our algorithm following the framework
of~\citep{sgmcmc_2015}.

\newpage
\bibliography{paper}
\bibliographystyle{plainnat}

\clearpage
\appendix


\onecolumn
\begin{center}
{\Large Appendix}
\end{center}

This appendix is divided into three major sections. Appendix~\ref{app:proofs} provides
the proofs that we omitted from the main text due to space constraints.
Appendix~\ref{app:why_barker} elaborates on our choice of the Barker logistic
function. Finally, Appendix~\ref{app:experiments} presents further details on the correction
distribution numerical derivation and on our three main experiments to assist
understanding and reproducibility.  


\section{PROOFS OF LEMMAS AND COROLLARIES}\label{app:proofs}

\subsection{PROOF OF LEMMA~\ref{lem:worst_case}}\label{app:worst_case_proof}

Choose $(\theta' - \theta) \in \pm\frac{1}{\sqrt{N}}[0.5,1]$ (event 1) and
$(\theta -0.5) \in \pm\frac{1}{\sqrt{N}}[0.5,1]$ filtered for  matching sign
(event 2).  As discussed in Lemma~\ref{lem:worst_case}, both $q(\theta' |
\theta)$ and $p(\theta | x_1,\ldots,x_N)$ have variance $1/N$. If we denote
$\Phi$ as the CDF of the standard normal distribution, then the former event
occurs with probability $p_0 = 2(\Phi(\sqrt{N}\frac{1}{\sqrt{N}}) -
\Phi(\sqrt{N}\frac{0.5}{\sqrt{N}})) = 2(\Phi(1)-\Phi(0.5)) \approx 0.2997$. The
latter event, because we restrict signs, occurs with probability $p_1 =
\Phi(1)-\Phi(0.5) \approx 0.14988$. 

These events together guarantee that $\Lambda^*(\theta,\theta')$ is negative by
inspection of Equation (\ref{eq:whyneg}) below.  This implies that we can find a
$u \in (0,1)$ so that $\psi(u,\theta,\theta') = \log u < 0$ equals
$\mE[\Lambda^*(\theta,\theta')]$.  Specifically, choose $u_0$ to satisfy $\log u_0
= \mE[\Lambda^*(\theta,\theta')]$.  Using $\mE[x_i^*] = 0.5$ and
Equation~(\ref{eq:lemma_ll_ratio}), we see that
\begin{equation}\label{eq:whyneg}
    \log u_0 = N(\theta'-\theta)\frac{1}{b} \cdot \mE\left[\sum_{i=1}^b x_i^*-\theta-\frac{\theta'-\theta}{2}\right]
    = -N(\theta'-\theta)\left(\theta-0.5+\frac{\theta'-\theta}{2}\right).
\end{equation}
Next, consider the minibatch acceptance test $\Lambda^*(\theta,\theta')
\not\approx \psi(u,\theta,\theta')$ used in ~\citep{cutting_mh_2014}
and~\citep{icml2014c1_bardenet14} , where $\not\approx$ means ``significantly
different from'' under the distribution over samples. This is
\begin{align}
\Lambda^*(\theta,\theta')  \not\approx \psi(u_0,\theta,\theta') 
&\iff N(\theta'-\theta) \cdot \frac{1}{b}\sum_{i=1}^b x_i^*-\theta-\frac{\theta'-\theta}{2} \not\approx \log u_0\\
&\iff \frac{1}{b}\sum_{i=1}^b x_i^*-\left(\theta+\frac{\theta'-\theta}{2} + \frac{\log u_0}{N(\theta'-\theta)}\right) \not\approx  0 \\
&\iff \frac{1}{b}\sum_{i=1}^b x_i^*-0.5 \not\approx 0. \label{eq:accept_test_zero}
\end{align}
Since the $x_i^*$ have mean 0.5, the resulting test with our chosen $u_0$ will
never correctly succeed and must use all $N$ data points.  Furthermore, if we
sample values of $u$ near enough to $u_0$, the terms in parenthesis will not be
sufficiently different from 0.5 to allow the test to succeed. 
  
The choices above for $\theta$ and $\theta'$ guarantee that
\begin{equation}\label{eq:log_uo_range}
    \log u_0 \in -[0.5,1][0.75,1.5] = [-1.5, -0.375].
\end{equation}
Next, consider the range of $u$ values near $u_0$:
\begin{equation}\label{eq:log_u_range}
    \log u \in \log u_0 + [-0.5,0.375].
\end{equation}
The size of the range in $u$ is at least $\exp([-2,-1.125]) \approx
[0.13534,0.32465]$ and occurs with probability at least $p_2=0.18932$. With $u$
in this range, we rewrite the test as:
\begin{equation}\label{eq:accept_test_rewritten}
    \frac{1}{b}\sum_{i=1}^b x_i^*-0.5 \hspace{0.1in} \not\approx \hspace{0.1in} \frac{\log u/u_0}{N(\theta'-\theta)}
\end{equation}
so that, as in Equation~(\ref{eq:accept_test_zero}), the LHS has expected value
zero.  Given our choice of intervals for the variables, we can compute the range
for the right hand side (RHS) assuming\footnote{If $\theta'-\theta<0$, then the
range would be $\frac{1}{\sqrt{N}}[-0.75,1]$ but this does not matter for
the purposes of our analysis.} that $\theta'-\theta > 0$:
\begin{equation}\label{eq:rhs_range}
\min\{{\rm RHS}\} = \frac{-0.5}{\sqrt{N} \cdot 0.5} = -\frac{1}{\sqrt{N}}
\quad {\rm and} \quad \max\{{\rm RHS}\} = \frac{0.375}{\sqrt{N} \cdot 0.5} = \frac{0.75}{\sqrt{N}}
\end{equation}
Thus, the RHS is in $\frac{1}{\sqrt{N}}[-1,0.75]$.  The standard deviation of
the LHS given the interval constraints is at least $0.5/\sqrt{b}$.
Consequently, the gap between the LHS and RHS in
Equation~(\ref{eq:accept_test_rewritten}) is at most $2\sqrt{b/N}$ standard
deviations, limiting the range in which the test will be able to ``succeed''
without requiring more samples.

The samples $\theta$, $\theta'$ and $u$ are drawn independently and so the
probability of the conjunction of these events is $c = p_0 p_1 p_2 = 0.0085$.

\subsection{PROOF OF LEMMA~\ref{lem:quant_clt}}\label{app:quant_clt}

The following bound is given immediately after Corollary 2 from~\citep{explicit-clt05}:
\begin{equation}
-6.4\mE[|X|^3]-2\mE[|X|] \leq \sup_x|{\rm Pr}(t<x)-\Phi(x)|{\sqrt{n}} \leq
1.36\mE[|X|^3].
\end{equation}
This bound applies to $x\geq 0$. Applying the bound to $-x$ when $x<0$
and combining with $x>0$, we obtain the weaker but unqualified bound
in Equation~(\ref{eq:clt-bounds}).

\subsection{PROOF OF LEMMA~\ref{lem:cdf_bounds}}\label{app:proof_cdf_bounds}

We first observe that
\[
    P'(z) - Q'(z) = \int_{-\infty}^{+\infty}(P(z-x)-Q(z-x))R(x) dx,
\]
and since $\sup_x|P(x)-Q(x)|\leq \epsilon$ it follows that $\forall z$:
\begin{equation}
-\epsilon = \int_{-\infty}^{+\infty} -\epsilon R(x) dx \leq \int_{-\infty}^{+\infty}(P(z-x)-Q(z-x))R(x) dx \leq \int_{-\infty}^{+\infty}\epsilon R(x) dx = \epsilon,
\end{equation}
as desired.

\subsection{PROOF OF COROLLARY~\ref{cor:bounds_preserved}}\label{app:bounds_preserved}

We apply Lemma~\ref{lem:cdf_bounds} twice. First take:
\begin{equation}
    P(y) = {\rm Pr}(\Delta^* < y)
     \quad {\rm and} \quad Q(y) = \Phi\left(\frac{y-\Delta}{s_{\Delta^*}}\right)
\end{equation}
and convolve with the distribution of $X_n$ which has density $\phi(X/\sigma_n)$
where $\sigma_n^2 = 1 - s^2_{\Delta^*}$. This yields the next iteration of $P$
and $Q$:
\begin{equation}
    P'(y) = {\rm Pr}(\Delta^*+X_{\rm nc} < y)
    \quad {\rm and}\quad Q'(y) = \Phi\left({y-\Delta}\right)
\end{equation}
Now we convolve with the distribution of $X_{\rm corr}$:
\begin{equation}
    P''(y) = {\rm Pr}(\Delta^*+X_{\rm nc}+X_{\rm corr} < y)
    \quad {\rm and}\quad Q''(y) = S\left({y-\Delta}\right)
\end{equation}
Both steps preserve the error bound $\epsilon(\theta,\theta',b)$. Finally
$S(y-\Delta)$ is a logistic CDF centered at $\Delta$, and so $S(y-\Delta) = {\rm
Pr}(\Delta + X_{\rm log} < y)$ for a logistic random $X_{\rm log}$. We conclude
that the probability of acceptance for the actual test ${\rm Pr}(\Delta^*+X_{\rm
nc}+X_{\rm corr} > 0)$ differs from the exact test ${\rm Pr}(\Delta+X_{\rm log}
> 0)$ by at most $\epsilon$.

\subsection{IMPROVED ERROR BOUNDS BASED ON SKEW ESTIMATION}\label{app:better_error_bound}

We show that the CLT error bound can be improved to $O(n^{-1})$ using a more
precise limit distribution under an additional assumption. Let $\mu_i$ denote
the $i^{th}$ moment, and $b_i$ denote the $i^{th}$ absolute moment of $X$. If
Cramer's condition holds:
\begin{equation}\label{eq:cramers_condition}
    \lim_{t \to \infty} \sup |\mE[\exp(i t X)]| < 1,
\end{equation}
then Equation 2.2 in Bentkus et al.'s work on Edgeworth
expansions~\citep{Bentkus97} provides:

\begin{lemma}\label{lem:clt_edgeworth}
Let $X_1,\ldots,X_n$ be a set of zero-mean, independent, identically-distributed
random variables with sample mean $\hat{X}$ and with $t$ defined as in Lemma 3.
If $X$ satisfies Cramer's condition, then
\[
    \sup_x\left|{\rm Pr}(t<x) - G\left(x, \frac{\mu_3}{b_2^{3/2}}\right)\right| \leq \frac{c(\epsilon,b_2,b_3,b_4,b_{4+\epsilon})}{n}
\]
where
\begin{equation}
    G_n(x,y) = \Phi(x) + \frac{y(2x^2+1)}{6\sqrt{n}}\Phi'(x).
\end{equation}
\end{lemma}
Lemma~\ref{lem:clt_edgeworth} shows that the average of the $X_i$ has a more
precise, skewed CDF limit $G_n(x,y)$ where the skew term has weight proportional
to a certain measure of skew derived from the moments:
$\mu_3/b_2^{3/2}$. Note that if the $X_i$ are symmetric, the weight of
the correction term is zero, and the CDF of the average of the $X_i$ converges
to $\Phi(x)$ at a rate of $O(n^{-1})$.

Here the limit $G_n(x,y)$ is a normal CDF plus a correction term that decays as
$n^{-1/2}$.
Importantly, since $\phi^{''}(x) = x^2\phi(x) - \phi(x)$ where
$\phi(x)=\Phi'(x)$, the correction term can be rewritten giving:
\begin{equation}\label{eq:GNderivatives}
    G_n(x,y) = \Phi(x) + \frac{y}{6\sqrt{n}}(2\phi^{''}(x)+3\phi(x))
\end{equation}
From which we see that $G_n(x,y)$ is a linear combination of $\Phi(x)$,
$\phi(x)$ and $\phi^{''}(x)$. In Algorithm 1, we
correct for the difference in $\sigma$ between $\Delta^*$ and the variance
needed by $X_{\rm corr}$ using $X_{\rm nc}$. This same method works when we
wish to estimate the error in $\Delta^*$ vs $G_n(x,y)$. Since all of the
component functions of $G_n(x,y)$ are derivatives of a (unit variance)
$\Phi(x)$, adding a normal variable with variance $\sigma'$ increases the
variance of all three functions to $1+\sigma'$. Thus we add $X_{\rm nc}$ as
per Algorithm 1 preserving the limit in Equation~(\ref{eq:GNderivatives}).

The deconvolution approach can be used to construct a correction variable
$X_{\rm corr}$ between $G_n(x,y)$ and $S(x)$ the standard logistic function. An
additional complexity is that $G_n(x,y)$ has additional parameters $y$ and $n$.
Since these act as a single multiplier $\frac{y}{6\sqrt{n}}$ in
Equation~(\ref{eq:GNderivatives}), its enough to consider a function $g(x,y')$
parametrized by $y'= \frac{y}{6\sqrt{n}}$. This function can be computed and
saved offline. As we have shown earlier, errors in the ``limit'' function
propagate directly through as errors in the acceptance test.  To achieve a test
error of $10^{-6}$ (close to single floating point precision), we need a $y'$
spacing of $10^{-6}$. It should not be necessary to tabulate values all the way to
$y'=1$, since $y'$ is scaled inversely by the square root of minibatch size.
Assuming a max $y'$ of 0.1 requires us to tabulate about 100,000.  Since our $x$
resolution is 10,000, this leads to a table with about 1 billion values, which
can comfortably be stored in memory.  However, if $g(x,y)$ is moderately smooth
in $y$, it should be possible to achieve similar accuracy with a much smaller
table. We leave further analysis and experiments with $g(x,y)$ as future work.

\section{WHY THE BARKER LOGISTIC FUNCTION?}\label{app:why_barker}

Regarding our choice of the Logistic function, a test function $f(x)$ for
Metropolis-Hastings must satisfy Lemma~\ref{lem:detailed_balance}. In addition,
it must be monotone, bounded by $[0,1]$ and be such that $\lim_{x \to
-\infty}f(x) = 0$ and $\lim_{x \to \infty}f(x) = 1$. While many functions
satisfy this, including the classical test $f(x) = \min\{\exp(x), 1\}$, the
Logistic function is the \emph{unique} function in this class which is
anti-symmetric about 0.5, so it represents the (unique) CDF of a symmetric
random variable. Our method requires approximating this with the sum of a
Gaussian random variable (which is symmetric) and a correction. The Logistic CDF
$L$ and Gaussian CDF $\Phi$ are extremely close even without correction; more
precisely, the CDF error from the closest Gaussian CDF --- which we numerically
determined to have standard deviation approximately 1.7 --- satisfies $\sup_x
|L(x) - \Phi(x/1.7)| < 0.01$. Said another way, the error between the Logistic
and Gaussian CDFs is less than 1\%. With our correction we can make this error
orders of magnitude smaller. 

While not a proof of optimality, it is unlikely that a non-symmetric test
function $f(x)$ --- representing a skewed variable --- would do better. It would
require a highly-skewed correction variable, and likely require a much narrower
normal distribution (and hence more samples).

\section{ADDITIONAL EXPERIMENT DETAILS}\label{app:experiments}

\subsection{OBTAINING THE CORRECTION DISTRIBUTION (SECTION~\ref{sec:correction})}\label{app:correction}

\begin{table}[t]
\caption{Errors ($L_\infty$) in $X_{\rm norm}+X_{\rm corr}$ versus $X_{\rm
log}$, with $N=4000$ (top row) and $N=2000$ (bottom row).}
\label{tab:xcorr}
\vskip 0.15in
\begin{center}
\begin{small}
\begin{tabular}{l l}
\hline
\T\B  $N=2000$ & $\sigma=0.8$ \rule{0pt}{3ex} \\
\hline
\T\B $\lambda$ & $L_{\infty}$ error \\
\hline
100    & 2.6e-3 \T \\
10     & 4.0e-4  \\
1      & 6.7e-5  \\
0.1    & 1.4e-5  \\
0.01   & \textbf{5.0e-6} \B \\
\hline
\end{tabular}
\quad 
\begin{tabular}{l l}
\hline
\T\B $N=2000$ & $\sigma=0.9$ \\
\hline
\T\B $\lambda$ & $L_{\infty}$ error \\
\hline
100    & 3.3e-3 \T \\
10     & 6.4e-4  \\
1      & 1.6e-4  \\
0.1    & \textbf{1.3e-4}  \\
0.01   & 2.7e-4 \B \\
\hline
\end{tabular}
\quad 
\begin{tabular}{l l}
\hline
\T\B $N=2000$ & $\sigma=1.0$ \\
\hline
\T\B $\lambda$ & $L_{\infty}$ error \\
\hline
100    & 4.4e-3 \T \\
10     & 1.3e-3  \\
1      & \textbf{1.1e-3}  \\
0.1    & 2.0e-3  \\
0.01   & 3.6e-3 \B \\
\hline
\end{tabular}
\quad 
\begin{tabular}{l l}
\hline
\T\B $N=2000$ & $\sigma=1.1$ \\
\hline
\T\B $\lambda$ & $L_{\infty}$ error \\
\hline
100    & 6.8e-3 \T \\
10     & \textbf{4.6e-3}  \\
1      & 7.5e-3  \\
0.1    & 1.3e-2  \\
0.01   & 2.4e-2 \B \\
\hline
\end{tabular}
\vskip 0.2in 
\begin{tabular}{l l}
\hline
\T\B $N=4000$ & $\sigma=0.8$ \\
\hline
\T\B $\lambda$ & $L_{\infty}$ error \\
\hline
100    & 8.3e-4 \T \\
10     & 1.3e-4  \\
1      & 2.5e-5  \\
0.1    & \textbf{6.7e-6}  \\
0.01   & 7.4e-6 \B \\
\hline
\end{tabular}
\quad 
\begin{tabular}{l l}
\hline
\T\B $N=4000$ & $\sigma=0.9$ \\
\hline
\T\B $\lambda$ & $L_{\infty}$ error \\
\hline
100    & 1.2e-3 \T \\
10     & 2.6e-4  \\
1      & \textbf{1.0e-4}  \\
0.1    & 2.0e-4  \\
0.01   & 3.9e-4 \B \\
\hline
\end{tabular}
\quad 
\begin{tabular}{l l}
\hline
\T\B $N=4000$ & $\sigma=1.0$ \\
\hline
\T\B $\lambda$ & $L_{\infty}$ error \\
\hline
100    & 1.9e-3 \T \\
10     & \textbf{8.9e-4}  \\
1      & 1.6e-3  \\
0.1    & 2.8e-3  \\
0.01   & 5.2e-3 \B \\
\hline
\end{tabular}
\quad 
\begin{tabular}{l l}
\hline
\T\B $N=4000$ & $\sigma=1.1$ \\
\hline
\T\B $\lambda$ & $L_{\infty}$ error \\
\hline
100    & \textbf{4.3e-3} \T \\
10     & 6.0e-3  \\
1      & 1.0e-2  \\
0.1    & 1.2e-2  \\
0.01   & 3.5e-2 \B \\
\hline
\end{tabular}
\end{small}
\end{center}
\vskip -0.1in
\end{table}

In Section~\ref{sec:correction}, we described our derivation of the correction
distribution $C_\sigma$ for random variable $X_{\rm corr}$.
Table~\ref{tab:xcorr} shows our $L_\infty$ error results for the convolution
(Equation~(\ref{eq:overall_corr_problem})) based on various hyperparameter
choices.  We test using $N=2000$ and $N=4000$ points for discretization within
a range of $X_{\rm corr} \in [-20,20]$, covering essentially all the probability
mass. We also vary $\sigma$ from 0.8 to 1.1. 

We observe the expected tradeoff. With smaller $\sigma$, our $C_\sigma$ is
closer to the ideal distribution (as judged by $L_\infty$ error), but this
imposes a stricter upper bound on the sample variance of $\Delta^*$ before our
test can be applied, which thus results in larger minibatch sizes. Conversely, a
more liberal upper bound means we avail ourselves of smaller minibatch sizes,
but at the cost of a less stable derivation for $C_\sigma$.

We chose $N=4000, \sigma=1$, and $\lambda = 10$ to use in our experiments, which
empirically exhibits excellent performance. This is reflected in the description
of {\sc MHminibatch} in Algorithm~\ref{alg:our_algorithm}, which assumes that we
used $\sigma=1$ but we reiterate that the choice is arbitrary so long as $0 <
\sigma < \sqrt{\pi^2/3} \approx 1.814$, the standard deviation of the standard
logistic distribution, since there must be some variance left over for $X_{\rm
corr}$.

\subsection{GAUSSIAN MIXTURE MODEL EXPERIMENT (SECTION~\ref{ssec:gaussians})}\label{app:gaussians}

\begin{table}[t]
\caption{Gaussian Mixture Model statistics ($\pm$ one standard deviation over 10
trials).}
\small
\label{tab:poissons}
\vskip 0.15in
\begin{center}
\begin{tabular}{l l l l}
\textbf{Metric/Method} & {\sc MHminibatch} & {\sc AustereMH(c)} & {\sc MHSubLhd} \\
\hline
\T Equation~\ref{eq:log_prob_poissons} & $-1307.0 \pm 229.5$ & $-1386.9 \pm 322.4$ & $-1295.1 \pm 278.0$ \\
Chi-Squared & $4502.3 \pm 1821.8$ & $ 5216.9 \pm 3315.8$ & $ 3462.3 \pm 1519.5$ \\
\end{tabular}
\end{center}
\vskip -0.1in
\end{table}

\subsubsection{Grid Search}\label{sssec:grid_search}

For the Gaussian mixture experiment, we use the conservative method
from~\citep{cutting_mh_2014}, which avoids the need for recomputing log
likelihoods of each data point each iteration by choosing baseline minibatch
sizes $m$ and per-test thresholds $\epsilon$ beforehand, and then using those
values for the entirety of the trials. We experimented with the following
values, which are similar to the values reported in~\citep{cutting_mh_2014}:
\begin{itemize}[noitemsep]
\item $\epsilon \in \{0.001, 0.005, 0.01, 0.05, 0.1, 0.2\}$
\item $m \in \{50, 100, 150, 200, 250, 300, 350, 400, 450, 500\}$
\end{itemize}
and chose the $(m,\epsilon)$ pairing which resulted in the lowest expected data
usage given a selected upper bound on the error. Through personal communication
with~\citet{cutting_mh_2014}, we were able to use their same code to compute
expected data usage and errors.

The main difference between {\sc AustereMH(c)} and {\sc
AustereMH(nc)}\footnote{{\sc AustereMH(nc)} is used in
Section~\ref{ssec:logistic}.}  is that the latter needs to run a grid search
each iteration (i.e. after each time it makes an accept/reject decision for one
sample $\theta_t$). We use the same $\epsilon$ and $m$ candidates above for {\sc
AustereMH(nc)}.

\subsubsection{Gaussian Mixture Model Metrics}\label{sssec:gauss_metrics}

We discretize the posterior coordinates into bins with respect to the two
components of $\theta$.  The probability $P_i$ of a sample falling into bin $i$
is the integral of the true posterior over the bin's area.  A single sample
should therefore be multinomial with distribution $P$, and a set of $n$ (ideally
independent) samples is ${\rm Multinomial}(P,n)$. This distribution is simple
and we can use it to measure the quality of the samples rather than use general
purpose tests like KL-divergence or likelihood-ratio, which are problematic with
zero counts.

For large $n$, the per-bin distributions are approximated by Poissons with
parameter $\lambda_i=P_i n$. Given samples $\{\theta_1,\ldots,\theta_T\}$, let
$c_j$ denote the number of individual samples $\theta_i$ that fall in bin $j$
out of $N_{\rm bins}$ total. We have
\begin{equation}\label{eq:log_prob_poissons}
\log p(c_1, \ldots, c_{N_{\rm bins}} | P_1, \ldots, P_{N_{\rm bins}}) =
\sum_{j=1}^{N_{\rm bins}} c_j \log (n P_j) - n P_j - \log(\Gamma(c_j+1)).
\end{equation}
Table~\ref{tab:poissons} shows the likelihoods. To facilitate interpretation we
perform significance tests using Chi-Squared distribution (also in
Table~\ref{tab:poissons}). The table provides the mean likelihood value and mean
Chi-Squared test statistics value as well as their standard deviations.  Our
likelihood values lies between~\citep{cutting_mh_2014}
and~\citep{icml2014c1_bardenet14}, but we note that we are not aiming to
optimize the likelihood values or the Chi-Squared statistics.  We use these
values to show the extent of correctness.

\subsection{LOGISTIC REGRESSION EXPERIMENT (SECTION~\ref{ssec:logistic})}\label{app:logistic}

\begin{figure*}[t]
	\centering
    \includegraphics[width=0.9\linewidth]{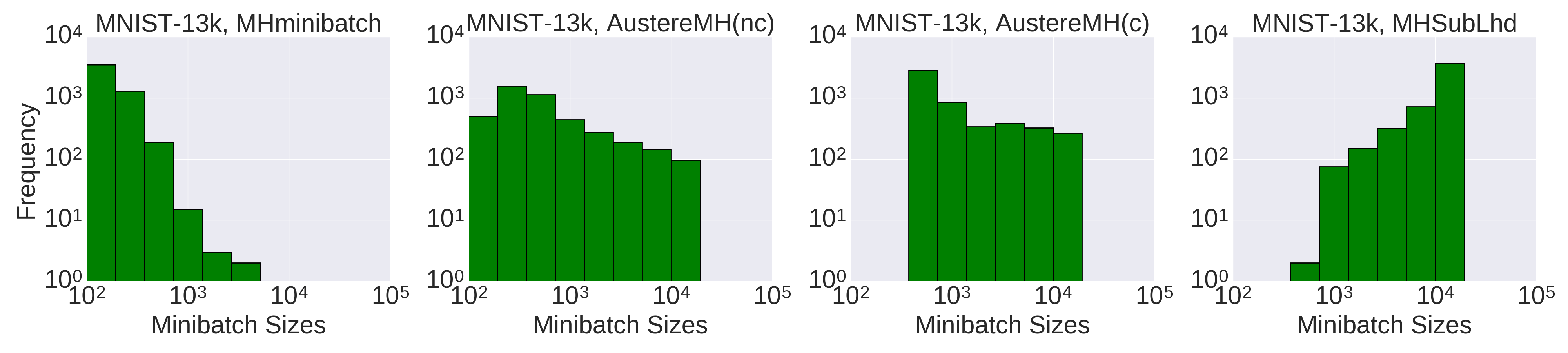}
	\caption{
    Minibatch sizes for a representative trial of logistic regression on
    MNIST-13k (analogous to Figure~\ref{fig:gauss_mix_2}). Both axes are on a
    log scale and have the same ranges across the three histograms. See
    Section~\ref{ssec:logistic} for details.
    }
	\label{fig:logistic_histograms_mnist}
\end{figure*}

\begin{figure*}[t]
	\centering
    \includegraphics[width=0.9\linewidth]{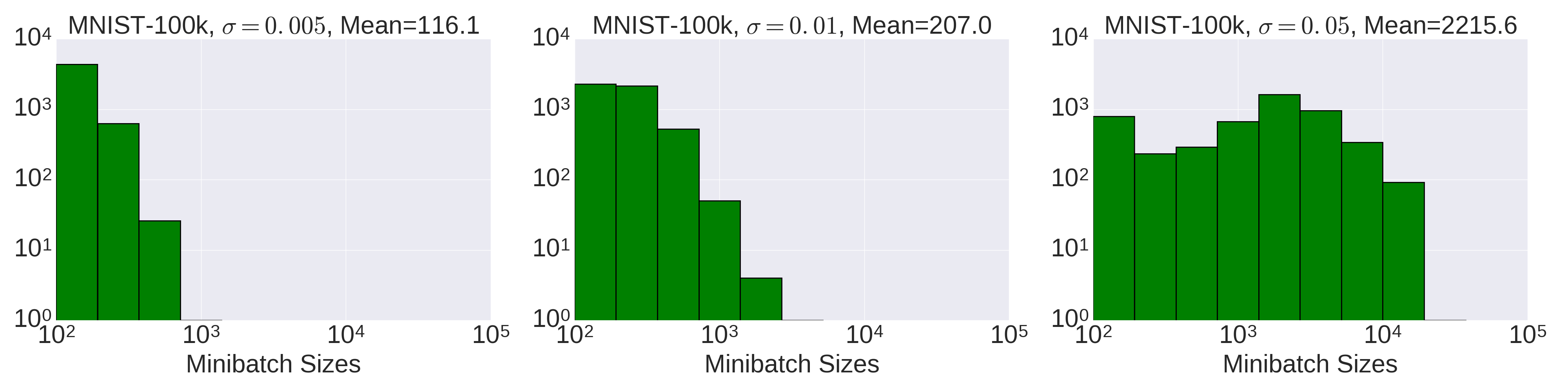}
	\caption{
    Effect of changing the proposal step size $\sigma$ for
    {\sc MHminibatch}.
    }
	\label{fig:logistic_histograms_appendix}
\end{figure*}

\begin{table}[t]
\caption{Parameters for the logistic regression experiments.}
\small
\label{tab:logistic_settings}
\vskip 0.15in
\begin{center}
\begin{tabular}{l l l}
\textbf{Value} & \textbf{MNIST-13k} & \textbf{MNIST-100k} \\
\hline
\T Temperature $K$       & 100 & 100 \\
Number of samples $T$ & 5000 & 3000 \\
Number of trials      & 10 & 5 \\
Step size $\sigma$ for random walk proposer with covariance $\sigma I$ & 0.05 & 0.01 \\
{\sc MHminibatch} and {\sc MHSubLhd} minibatch size $m$ & 100 & 100 \\
{\sc AustereMH(c)} chosen $\Delta^*$ bound & 0.1 & 0.2 \\
{\sc AustereMH(c)} minibatch size $m$ from grid search & 450 & 300 \\
{\sc AustereMH(c)} per-test threshold $\epsilon$ from grid search & 0.01 & 0.01\\
{\sc AustereMH(nc)} chosen $\Delta^*$ bound & 0.05 & 0.1 \\
{\sc MHSubLhd} $\gamma$ & 2.0  & 2.0  \\
{\sc MHSubLhd} $p$      & 2 & 2 \\
{\sc MHSubLhd} $\delta$ & 0.01 & 0.01 \\
\end{tabular}
\end{center}
\vskip -0.1in
\end{table}

Figure~\ref{fig:logistic_histograms_mnist} shows the histograms for the four
methods on one representative trial of MNIST-13k, indicating similar relative
performance of the four methods as in
Figure~\ref{fig:logistic_histograms_mnist8m} (which uses MNIST-100k). In
particular, {\sc MHminibatch} exhibits a shorter-tailed distribution and
consumes nearly an order of magnitude fewer data points compared to {\sc
AustereMH(nc)}, the next-best method; see Table~\ref{tab:logistic} for details.

Next, we investigate the impact of the step size $\sigma$ for the random walk
proposers with covariance matrix $\sigma I$. Note that $I$ is $784\times 784$ as
we did not perform any downsampling or data preprocessing other than rescaling
the pixel values to lie in $[0,1]$.

For this, we use the larger dataset MNIST-100k, and test with $\sigma \in
\{0.005, 0.01, 0.05\}$. We keep other parameters consistent with the experiments
in Section~\ref{ssec:logistic}, in particular, keeping the initial minibatch
size $m=100$, which is also the amount the minibatch increments by if we need
more data. Figure~\ref{fig:logistic_histograms_appendix} indicates minibatch
histograms (again, using the log-log scale) for one trial of {\sc MHminibatch}
using each of the step sizes. We observe that by tuning {\sc MHminibatch}, we
are able to adjust the average number of data points in a minibatch across a
wide range of values. Here, the smallest step size results in an average of just
116.1 data points per minibatch, while increasing to $\sigma=0.05$ (the step
size used for MNIST-13k) results in an average of 2215.6. This relative trend is
also present for both {\sc AustereMH} variants and {\sc MHSubLhd}.

Table~\ref{tab:logistic_settings} indicates the relevant parameter settings for
the logistic regression experiments. Unless otherwise stated, values apply to
all methods tested. For values from~\citep{cutting_mh_2014}
or~\citep{icml2014c1_bardenet14}, we use their notation ($\Delta^*, m, \epsilon,
\gamma, p$, and $\delta$) to be consistent.

\end{document}